\documentclass[runningheads]{llncs}
\usepackage{amsmath,amsfonts,bm}

\def\eqref#1{equation~\ref{#1}}

\def\1{\bm{1}}

\def\vzero{{\bm{0}}}
\def\vone{{\bm{1}}}

\def\va{{\bm{a}}}
\def\vb{{\bm{b}}}
\def\vc{{\bm{c}}}

\def\vp{{\bm{p}}}
\def\vq{{\bm{q}}}
\def\vr{{\bm{r}}}

\def\vx{{\bm{x}}}

\DeclareMathAlphabet{\mathsfit}{\encodingdefault}{\sfdefault}{m}{sl}
\SetMathAlphabet{\mathsfit}{bold}{\encodingdefault}{\sfdefault}{bx}{n}

\usepackage[T1]{fontenc}
\usepackage{xurl}
\usepackage{subcaption}
\usepackage{graphicx}
\usepackage{booktabs}
\usepackage{stmaryrd}
\usepackage{multirow}
\usepackage{adjustbox}
\usepackage{hyperref}
\usepackage{xcolor}
\usepackage{todonotes}
\providecommand{\Description}[1]{}

\newcommand{\lowc}[1]{\mbox{low}({#1})}
\newcommand{\upc}[1]{\mbox{up}({#1})}
\newcommand{\cen}[1]{\mbox{cen}({#1})}
\newcommand{\off}[1]{\mbox{off}({#1})}
\newcommand{\std}[1]{{\scriptsize $\pm$#1}}
\newcommand{\f}[1]{\textbf{#1}}
\newcommand{\s}[1]{\underline{#1}}
\newcommand{\g}[1]{\textcolor{gray}{#1}}
\newcommand{\method}{GeometrE}
\newcommand{\betae}{{\normalsize{B}\footnotesize{ETA}\normalsize{E}}}
\newcommand{\relemb}[1]{\ensuremath{\bm{#1} = (#1_1, #1_2, #1_3, #1_4)}}
\newcommand{\proj}[2]{\hat{#1}_{#2}}
\newcommand{\dist}[1]{\mbox{dist}_{\scriptsize \mbox{#1}}}
\newcommand{\query}[1]{#1_{\mbox{\scriptsize query}}}
\newcommand{\answer}[1]{#1_{\mbox{\scriptsize answer}}}
\renewcommand{\paragraph}[1]{\subsubsection{#1}}

\begin{document}
\title{Fully Geometric Multi-Hop Reasoning on Knowledge Graphs With
  Transitive Relations}
\titlerunning{GeometrE}
%
\author{Fernando Zhapa-Camacho\inst{1,2,3}\orcidID{0000-0002-0710-2259} \and \\
  Robert Hoehndorf\inst{1,2,3}\orcidID{0000-0001-8149-5890}
}

\authorrunning{F. Zhapa-Camacho and R. Hoehndorf}
%

\institute{Computer, Electrical and Mathematical Sciences \&
  Engineering Division, King Abdullah University of Science and
  Technology, 4700 KAUST, 23955, Thuwal, Saudi Arabia \and
  KAUST Center of Excellence for Smart Health (KCSH), King Abdullah
    University of Science and Technology, 4700 King Abdullah
    University of Science and Technology, Thuwal, Saudi Arabia \and
    KAUST Center of Excellence for Generative AI, King Abdullah
    University of Science and Technology, 4700 King Abdullah
    University of Science and Technology, Thuwal, Saudi Arabia \\
    \email{\{fernando.zhapacamacho,robert.hoehndorf\}@kaust.edu.sa}
  }

\maketitle              
\begin{abstract}
  Multi-hop logical reasoning on knowledge graphs requires faithfully
  mapping the logical semantics to latent space. Current geometric
  embedding methods show to be useful on this task by mapping entities
  to geometric regions and logical operations to latent
  transformations. While a geometric embedding can provide a direct
  interpretability framework for query answering, current methods have
  only leveraged the geometric construction of entities, failing to
  map logical operations to pure geometric transformations and,
  instead, using neural components to learn these operations. On the
  other hand, purely neural-based methods outperform geometric
  methods, but they lack interpretability in the latent space. We
  introduce \method{}, a geometric embedding method for multi-hop
  reasoning, that maps every logical operation to a purely geometric
  operation in the latent space. Additionally, we introduce a
  transitive loss function and show that, unlike existing methods, it
  can preserve the logical rule
  $\forall a,b,c: r(a,b) \land r(b,c) \to r(a,c)$. Our experiments
  show that \method{} outperforms current state-of-the-art geometric
  methods and remains competitive with existing neural-based methods
  on standard benchmark datasets.

\end{abstract}
\section{Introduction}
Different forms of information can be encoded in a graph-structure
manner and large knowledge graphs are found in different application
scenarios~\cite{Hogan_2021}. Knowledge graph completion involves
discovering facts that are not in a knowledge graph but may be added
to it~\cite{Peng_2023}. Multi-hop reasoning is the task of retrieving
answers for complex queries in knowledge graphs. While symbolic
methods can be employed for this task, they are limited when knowledge
graphs contain noise or are incomplete. Therefore, embedding-based
methods were developed to generate vector representations of queries
and their answers, and use them to answer multi-hop
queries.

Complex queries involve first order logical operations such as
intersection ($\land$), union ($\lor$), negation ($\neg$) and
existential quantification ($\exists$). Existing embedding methods map
queries, answers and logical operators into the embedding space, and
rely on geometric and probabilistic spaces to learn the
embeddings. Geometric methods map queries to points~\cite{gqe} or
closed regions~\cite{query2box,cone,acone} and associate potential
answers via distance or membership functions, respectively. In terms
of logical operations, methods such as Query2Box uses a box
model~\cite{query2box}, where intersection is a closed operation. To
further incorporate other operators like negation, more recent methods
leveraged other geometries/spaces such as cones~\cite{cone,acone} or
probability distributions~\cite{betae}.

Knowledge graphs represent relationships between entities. Modeling
relation properties (e.g., symmetry, invertibility, transitivity) is
another line of research when generating embeddings for knowledge
graphs. In tasks such as link prediction (1-hop queries) there is a
wide range of methods aiming to preserve relational properties
\cite{Jin_2023}; however, in the case of multi-hop reasoning, only
recent studies have paid attention to relation
properties~\cite{acone}.

We focus on two limitations of current work:

\paragraph{Geometric interpretability.} Interpretability refers to the
degree to which a model's internal mechanisms and features can be
understood \cite{Garouani_2024}.
Existing methods to model queries rely purely on neural networks to
model both queries and operators~\cite{long-etal-2022-neural}, or
provide a geometric construction of queries but use neural networks to
learn the logical operations~\cite{query2box,betae,cone}. In the
second case, the interpretability degree is higher, since we can
better understand the mechanism to represent queries. But, in both
cases interpretability of the logical operators is limited as such
operators are \emph{learned} by a neural network.

\paragraph{Modeling of transitive relations.} Transitive relations can
be understood in two ways: (a) the rule form
$r \circ \ldots \circ r \equiv r$, where $r$ composed with itself is
equivalent to $r$ and (b) the logical form
$\forall a,b,c: r(a,b) \land r(b,c) \to r(a,c)$ where entities $a,b,c$
exist in chain-like form under the relation $r$. Based on the rule
form, we can model the relation projection as an idempotent
transformation~\cite{rotpro} and, from the perspective of the logical
form, the chain-like structure of entities can be modeled with an
order-preserving loss function such as those used to preserve
hierarchies~\cite{ordere}.

While some multi-hop reasoning methods can encode relation
composition~\cite{query2box,acone}, they cannot model transitivity
(which is a special case of composition) because their encoding does
not preserve the idempotency property of transitive
relations. Furthermore, current methods rely on distance-based
functions to associate queries and answers, and those methods cannot
encode the chain-like form of transitive relations either.
\\

To overcome the aforementioned limitations, we present \method{}, a
box model for query answering. In \method{}, we map every logical
operation to a purely geometric transformation, therefore removing the
need of neural networks and making the embedding process fully
geometrically interpretable. Additionally, we construct an idempotent
geometric operation and propose a transitive loss function to preserve
transitivity of relations in the embedding space.  Our contributions
are listed below:
\begin{itemize}
\item We propose a geometrically interpretable embedding method based
  on box embeddings for query answering on knowledge
  graphs.
 \item We introduce a
   multiplicative-additive projection mechanism (combining additive and
    multiplicative components) that significantly enhances the
    expressiveness of relation embeddings compared to purely additive
    translations.
\item  We design a transitive loss function that captures transitivity
  in the embedding space.
\item We show empirical results across several standard benchmarks
  demonstrating that our method can outperform state-of-the-art
  geometric approaches and closing the gap to neural-based methods.
\end{itemize}

\section{Related Work}

\paragraph{Geometric embeddings:} Initial work on geometric embeddings
used points to model queries and answers~\cite{gqe}. Since queries can
contain more than one answer, modeling queries as closed
regions~\cite{query2box,cone,acone} improved the embedding quality.
Boxes~\cite{query2box} and cones~\cite{cone,acone} have been used to
model queries. Fundamentally, box models cannot represent negation,
since the complement of a box is not a box. Similarly, cone models
cannot exactly represent intersection because intersection of cones
might produce a set of disconnected cones~\cite{acone}. We have chosen
a box model to encode queries, and implement negation as a geometric
exclusion constraint, which works in the case `A but not B', which is
the negation form typically found in knowledge graph queries.

\paragraph{Embeddings for transitive relations.} Knowledge graph
embedding methods that encode transitive relations such as
RotPro~\cite{rotpro} models relations as rotations and implement an
idempotent transformation by projection.  Hierarchies, which are
transitive, have been modeled by several approaches, such as using
concentric circles~\cite{hake} or hyperbolic spaces
~\cite{nickel2017poincare}. However, these methods are tailored for
link prediction (1-hop queries) and do not support more complex
queries because they do not provide mechanisms to compute
intersection, negation, or union of entities.

\paragraph{Non-geometric embeddings.} Other types of methods have been
proposed to address the multi-hop reasoning problem. Non-geometric
methods use neural approaches to model logical
operations~\cite{long-etal-2022-neural} or use Large Language Models
(LLMs) to enhance query and entity representations~\cite{Phan_2025}.
We exclude LLM-based methods because they operate on a fundamentally
different paradigm. While LLMs use pre-trained external knowledge
(mostly natural language), our method relies strictly on the graph
topology and logical structure. Therefore, we consider LLM-based
reasoning out of scope, as our objective is to map logical semantics
to a geometric latent space to enable interpretable reasoning.

\section{Preliminaries}

\paragraph{Knowledge Graphs. } A knowledge graph is a tuple
$(\mathcal{V}, \mathcal{E}, \mathcal{R})$, where $\mathcal{V}$ is a
set of nodes, $\mathcal{R}$ is a set of relation labels and
$\mathcal{E} \subseteq \mathcal{V} \times \mathcal{R} \times
\mathcal{V}$ is a set of triples. Every element
$(s_i,r_j,o_k) \in \mathcal{V} \times \mathcal{R} \times \mathcal{V}$
can be seen as a logical binary predicate $r_j(s_i,o_k)$.

\paragraph{Logical queries.} Query embeddings in knowledge graphs are
represented as logical expressions using First-Order logic and involve
intersection ($\land$), union ($\lor$), negation ($\neg$) and
existential quantification ($\exists$).  A first-order logic query $q$
consists of a given anchor entity set
$\mathcal{V}_a \subseteq \mathcal{V}$, existentially quantified bound
variables $V_1,\ldots,V_k$ and a single target variable $V_?$ which
corresponds to the query answer. We use queries in Disjunctive Normal
Form (DNF) of the form:
\[
  q[V_?] = V_?:\exists V_1, \ldots, V_k: c_1 \lor \ldots \lor c_n
\]
Expressions such as $c_i$ are conjunctions
$c_i=e_{i1} \land \ldots \land e_{im}$, where $e_{ij}$ represents a
formula or its negation. $e_{ij}=r(v_a, V) $ or $ \neg (v_a, V)$ or
$r(V', V)$ or $\neg r(V',V)$ and $v_a \in \mathcal{V}_a$,
$V\in \{V_?, V_1, \ldots, V_k\}$, $V'\in \{V_1, \ldots, V_k\}$,
$V\neq V'$, $r \in \mathcal{R}$.

\paragraph{Relational properties}
Logical predicates in $\mathcal{R}$ can have properties such as:
symmetry, invertibility, composition or transitivity.
A relation $r \in \mathcal{R}$ is symmetric if
$\forall a,b \in \mathcal{V}: r(a,b) \to r(b,a)$. Given relations
$r,s \in \mathcal{R}$, $r$ is the inverse of $s$ if
$\forall a,b \in \mathcal{V}, r(a,b) \to s(b,a)$. Given relations
$r,s,t$, we say $t$ is the composition of $r$ and $s$ if
$\forall a,b,c \in \mathcal{V}: r(a,b) \land s(b,c) \to t(a,c)$. A
special case of composition is transitivity, where
$\forall a,b,c \in \mathcal{V}: r(a,b) \land r(b,c) \to r(a,c)$.

\begin{figure}[tbp]
  \Description{Representation of logical operations in GeometrE}
    \centering
    \begin{subfigure}[b]{0.14\textwidth}
      \includegraphics[width=\textwidth]{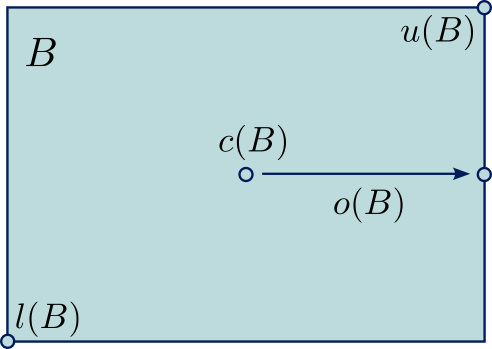}
      \caption{Box}
      \label{fig:box}
    \end{subfigure}
    \hfill
    \begin{subfigure}[b]{0.18\textwidth}
      \includegraphics[width=\textwidth]{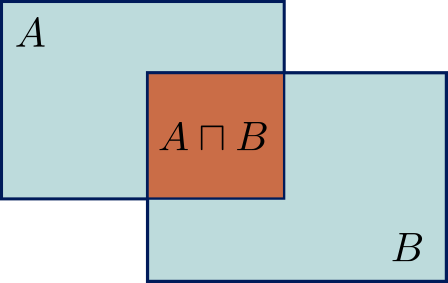}
      \caption{Intersection}
      \label{fig:intersection}
    \end{subfigure}
    \hfill
    \begin{subfigure}[b]{0.18\textwidth}
      \includegraphics[width=\textwidth]{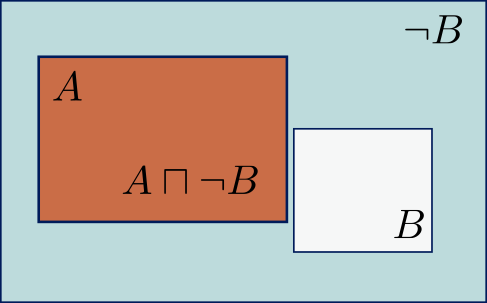}
      \caption{Negation}
      \label{fig:negation}
    \end{subfigure}
    \hfill
    \begin{subfigure}[b]{0.18\textwidth}
      \includegraphics[width=\textwidth]{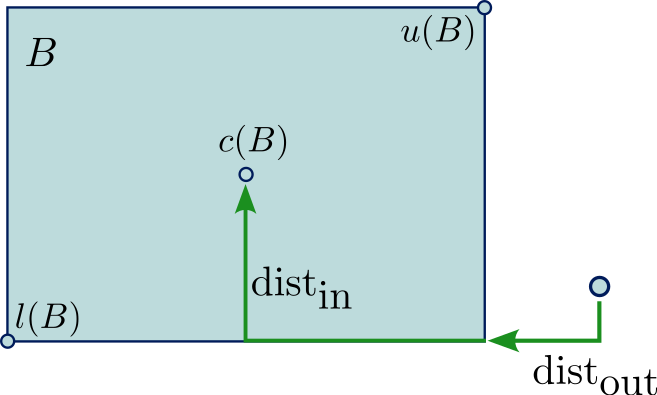}
      \caption{Point inclusion}
      \label{fig:inclusion}
    \end{subfigure}
    \hfill
    \begin{subfigure}[b]{0.19\textwidth}
      \includegraphics[width=\textwidth]{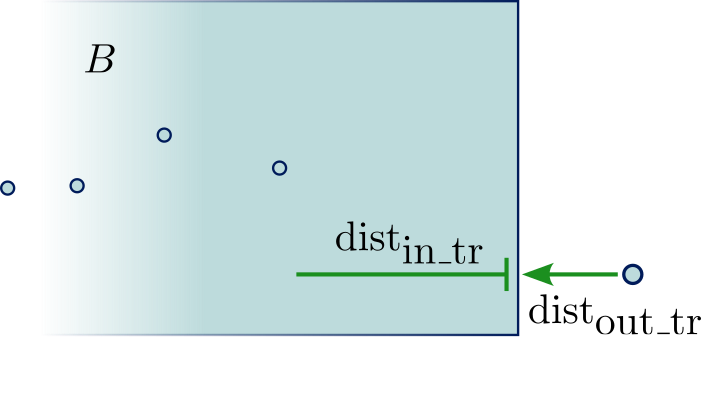}
      \caption{Point ordering}
      \label{fig:ordering}
    \end{subfigure}
    \caption{\method{} representation of boxes and logical operations in $\mathbb{R}^2$.}
    \label{fig:box_transformations}
\end{figure}

\section{Methodology}

We revisit the notion of box embeddings presented in
Query2Box~\cite{query2box} and propose \method{}, a box embedding
model for queries. We acknowledge that box embeddings are not closed
under negation and this limitation was addressed by using other
geometric spaces~\cite{betae,cone}. However, we model negation as a
geometric exclusion constraint. Our proposed constraint works for the
case `A but not B', which is the negation form typically found in
knowledge graph queries.

\paragraph{Box.} An $n$-dimensional axis-aligned box $B$ is a pair
$(\lowc{B}, \upc{B})$, where $\lowc{B}, \upc{B} \in \mathbb{R}^n$ and
$\lowc{B}[i] \leq \upc{B}[i], 1\leq i \leq n$. We call $\lowc{B}, $
and $\upc{B}$ the \emph{lower} and \emph{upper} corners of $B$,
respectively. Alternatively, a box $B$ can be defined as the pair
$(\cen{B}, \off{B})$, denoting the center and offset of $B$:
\begin{equation}
  \label{eq:box}
  \cen{B} = \frac{\upc{B}+\lowc{B}}{2}, \quad   \off{B} = \frac{\upc{B}-\lowc{B}}{2},
\end{equation}
where $\upc{B}, \lowc{B}, \cen{B}, \off{B} \in \mathbb{R}^n$. A box
$E$ has zero-volume when $\off{E}=\vzero$ (Figure \ref{fig:box}).

\paragraph{Geometric projection.} To model a relation
$r \in \mathcal{R}$, we represent its embedding as a 4-tuple
$\relemb{r}$. Given a box embedding $B=(\cen{B}, \off{B})$, we denote
as $B' = T_\vr(B)$ the transformation with respect to the 4-tuple
$\vr$ that generates the box $B'=(\cen{B'}, \off{B'})$:
\begin{equation}
  \label{eq:projection}
  \cen{B'} = r_1\otimes \cen{B} + r_2, \quad \off{B'} = |r_3 \otimes \off{B} + r_4|
\end{equation}
where $\otimes$ is coordinate-wise product between vectors. While
methods such as Query2Box~\cite{query2box} use a projection of the
form $\vr = (\vone, v_2, \vone, v_4)$ (with only additive components),
we implement a more expressive transformation by incorporating
multiplicative components, which improves the quality of embeddings
while maintaining geometric interpretability. The transformation $T_r$
(i.e., $T(r) = (T \circ T)(r)$) is idempotent in two cases: when
$T_r = id$ (the identity function) and when
$T(r)=(\vzero, v_2, \vzero, v_4)$ (the constant function).  For
formulas $r(A,B)$ and $\neg r(A,B)$, we use different relation
embeddings for $r$ and $\neg r$.

\paragraph{Geometric Intersection.} Given a set of boxes
$\{B_1, \ldots, B_n\}$, we compute the intersection box $I$ as:
\begin{equation}
  \label{eq:intersection}
  I = (\max(\lowc{B_1},\ldots, \lowc{B_n}), \min(\upc{B_1},\ldots, \upc{B_n}))
\end{equation}
where $\min$ and $\max$ operations are performed coordinate-wise
(Figure \ref{fig:intersection}).

 Prior works replaced geometric intersection with neural networks
  (DeepSets) to avoid vanishing gradients when boxes are disjoint
  \cite{Mei_2024,boxcd,improving_ident}. In a pure volume-based
  overlap, disjoint boxes yield zero signal. However, our method
  remains stable because our optimization objective
  (Eq. \ref{eq:inclusion}) minimizes the distance between the query
  box corners and the answer entity, rather than maximizing box
  overlap volume.  Even if two parent boxes $B_1$ and $B_2$ are
  disjoint, the resulting intersection box $I$ (calculated via
  coordinate-wise max and min) produces coordinates that are
  effectively `inverted' in the disjoint dimensions. The distance
  function then remains differentiable with respect to these
  corners. Consequently, the backpropagation signal effectively pulls
  the boundaries of the parent boxes $B_1$ and $B_2$ towards the
  answer $\va$, forcing the boxes to move until they overlap.
  
  In terms of complexity, processing each $n$-dimensional box using,
  for instance, an multi-layer perceptron (MLP) with $L$ layers, input
  size $n$ and hidden size $d$ requires $O(n\cdot d + L \cdot d^2)$
  operations and the complexity of computing the intersection for $k$
  boxes is $O(k\cdot (n \cdot d + L\cdot d^2))$. Cones are represented
  as two $n$-dimensional vectors~\cite{cone,acone}, the complexity of
  computing intersection has the same bounds as Query2Box. On the
  other hand, following the geometric definition of
  Equation~\ref{eq:intersection} requires $O(k\cdot n)$ operations.
  The geometric definition also provides direct geometric
  interpretability, which is not the case with neural networks used in
  previous works~\cite{gqe,query2box,betae,cone,acone}.

\paragraph{Geometric negation.}
The geometric complement of a box is not a box (it is the union of
open half-spaces), preventing closed-form logical negation. However,
in multi-hop reasoning, negation rarely appears in isolation (e.g.,
`give me everything that is not a Person'); it typically appears as a
constraint upon a positive query (e.g., `European countries that are
not in the EU').

  We model this specific structure not as a set complement, but as a
  geometric exclusion constraint. Given a positive query box $B_p$ and
  a negative query box $B_n$, we model the query $B_p \land \neg B_n$
  by enforcing the answer embeddings to minimize distance to $B_p$ while
  maximizing the distance from the interior of $B_n$. This avoids the
  need to explicitly represent the non-box shape of $\neg B_n$ in the
  latent space.

While this design does not
support absolute complement, it effectively models the negation form
typical in KG queries (e.g., `A but not B'), making it highly
applicable to standard benchmarks.

\subsection{Learning \method{} embeddings}

\begin{figure}[tbp]
  \Description{Query patterns used for training and evaluation}
  \centering
  \includegraphics[width=0.9\textwidth]{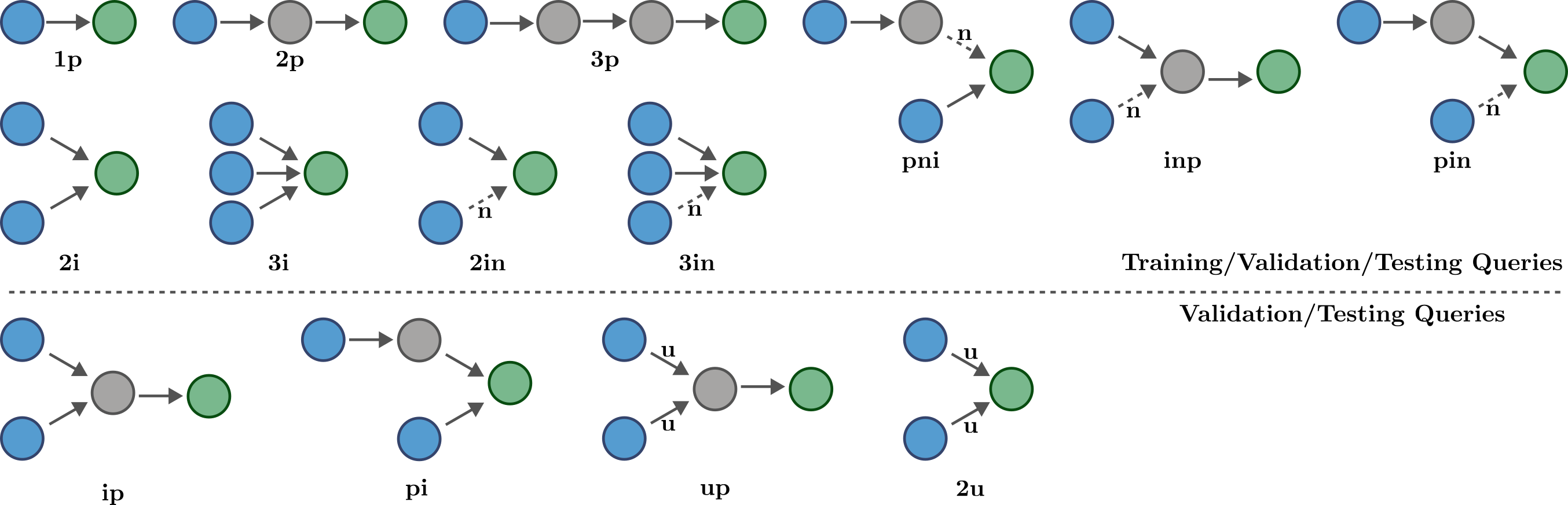}
  \caption{Query types supported by \method{}. Blue nodes represent
    anchor entities. Gray nodes represent intermediate operations
    (relation projection, intersection, union, negation) and green
    nodes represent answer entities.}
  \label{fig:query_types}
\end{figure}

\paragraph{Standard distance function.} The geometric construction of
\method{} supports queries shown in
Figure~\ref{fig:query_types}. Every query anchor node $q$ (blue nodes
in Figure~\ref{fig:query_types}) is associated to a box embedding
$\vq=(\cen{q}, \off{q})$, where $\cen{q}, \off{q} \in \mathbb{R}^n$,
and an answer node $a$ (green nodes in Figure~\ref{fig:query_types})
are associated to a zero-volume box
$\va=(\cen{a},\vzero)$. Intermediate nodes (gray nodes in
Figure~\ref{fig:query_types}) are queries constructed using relation
projection, intersection, negation, or union operations.

A node in the graph can appear both as anchor or answer node. We
differentiate between these two cases in the embedding space and
associate a node $a$ in the graph to a query embedding
$\query{\va}=(\cen{a}, \off{a})$ and to an answer embedding
$\answer{\va} =(\cen{a}', \vzero)$. As part of the optimization, we
enforce $\dist{box} (\query{\va}, \answer{\va}) = 0$.

In the end, to associate a query embedding $\vq$ to every answer
embedding $\va$ in the answer set $\llbracket q \rrbracket$, we use
the function:
\begin{equation}
  \label{eq:inclusion}
  \dist{box}(\vq, \va) = \dist{out}(\vq,\va) + \alpha \cdot \dist{in}(\vq,\va)
\end{equation}
where $\alpha$ is a hyperparameter that helps to enforce $\va$ to be
inside $\vq$ but not necessarily be equal to $\cen{q}$, and:
\begin{align}
  \dist{out}(\vq, \va) &= ||\max(\cen{a}' - \upc{q}, \vzero) + \max(\lowc{q} - \cen{a}',\vzero)||_1\\
  \dist{in}(\vq, \va) &= ||\cen{q} - \min(\upc{q}, \max(\lowc{q}, \cen{a}'))||_1
\end{align}
vectors $\upc{q}, \lowc{q}$ are computed from $\cen{q}$ and $\off{q}$
following Equation~\ref{eq:box}.

\paragraph{Negation distance function}

For queries that include negation, we rely on the assumption that the
positive and negative components of the query must be disjoint. Thus,
we use Equation~\ref{eq:inclusion} to compute the distance to the
positive component. For the negative component of the query, we
introduce the following exclusion distance function:

\begin{equation}
  \label{eq:exclusion}
  \dist{exclusion}(\vq, \va) = -\dist{out}(\vq,\va) + \alpha \cdot \dist{in}(\vq,\va)
\end{equation}

Unlike the standard distance function, $\dist{out}$ is negated, which
enforces the answer to be outside the query box. However, we keep
$\dist{in}$ positive with a regularization weight $\alpha$ to prevent
the answer embedding from moving arbitrarily far from the query region.

\paragraph{Order-preserving function}

Additionally, we incorporate an order-preserving function to account
for the transitivity property of some relations. A transitive relation
$r$ is such that $\forall a,b,c: r(a,b) \land r(b,c) \to r(a,c)$
(logic form) or $r\circ \ldots \circ r \equiv r$ (rule form). In both
cases, the transformation $T_r$ must be idempotent. Additionally, from
the logic form it follows that if
$\dist{box}(T_r(\va_{\mbox{\scriptsize query}}),\vb_{\mbox{\scriptsize
    answer}})=0$ and
$\dist{box}(T_r(\vb_{\mbox{\scriptsize query}}),\vc_{\mbox{\scriptsize
    answer}})=0$ then
$\dist{box}(T_r(\va_{\mbox{\scriptsize query}}),\vc_{\mbox{\scriptsize
    answer}})=0$, or approximates
to $0$. However, following Equation~\ref{eq:inclusion}, the embedding
model might collapse because it cannot meaningfully represent
$a,b,c,r$ at the same time. For example, w.l.o.g., assume
$\off{a} = \off{b} = \off{c} = \vzero$; then, with $\relemb{r}$ and
$r_4=\vzero$, we have that Equation~\ref{eq:inclusion} is minimized
when $r_1\otimes \cen{a} + r_2=\cen{b}$,
$r_1\otimes \cen{b} + r_2=\cen{c}$,
$r_1\otimes \cen{a} + r_2=\cen{c}$.

Therefore, to encode transitivity, we use the following
distance function for transitive relations:

\begin{equation}
  \label{eq:transitive_inclusion}
  \dist{box\_tr}(\vq, \va) = \dist{box}(\proj{\vq}{i},\proj{{\va}}i) +
  \dist{ordering}(\vq[i],\va[i])
\end{equation}
where $\proj{\vq}{i},\proj{\va}{i}$ are projected versions of boxes
where both centers and offsets have the value at coordinate $i$
removed. Additionally, at dimension $i$, we enforce an order embedding
\begin{equation}
  \label{eq:ordering}
  \dist{ordering}(\vq[i], \va[i]) =  ||\max(\cen{a}'[i] - \lowc{q}[i] + \lambda, 0) ||_1
\end{equation}

The idea behind $\dist{ordering}$ is to preserve the ordered structure
of transitive chains. Thus, for our previous example, $\dist{box\_tr}$
might collapse in all dimensions except $i$, where $\dist{ordering}$
will be minimized when $\cen{a} > \cen{b}'$, $\cen{b} > \cen{c}'$ and
$\cen{a} > \cen{c}'$. In the extreme case that query boxes have volume
$0$, we have that $\cen{a} > \cen{b} > \cen{c}$. The parameter
$\lambda$ ensures a minimum distance between entities.  Each
transitive relation is assigned a different projection dimension. In
the case that a transitive relation and its inverse (also transitive)
are present in the graph, the same dimension is used for both
relations and the distance function is modified to:

\begin{equation}
  \label{eq:ordering_inverse}
  \dist{ordering\_inverse}(\vq[i], \va[i]) =  ||\max(\upc{q}[i] -
  \cen{a}'[i] + \lambda, 0) ||_1
\end{equation}

Since the function $\dist{ordering}$ reserves specific dimensions of
$\mathbb{R}^n$ for ordering, the number of transitive relations $m$ is
constrained by $m\leq n$.

\paragraph{Training Objective.} Given a query $q$, our training
objective aims to minimize distance between the \method{} embedding of
$q$ and its answers while maximizing distance between $q$ and randomly
sampled negative answers~\cite{rotate}:
\begin{equation}
  \label{eq:loss}
  L = -\log \sigma (\gamma - \dist{box\_tr}(\vq,\va)) -
  \frac{1}{k}\sum_{i=1}^k \log
  \sigma(\dist{box\_tr}(\vq,\va'_i)-\gamma)
\end{equation}
where $\gamma>0$ is margin hyperparameter, $k$ is the number of
negative samples, $\va'_i $ is the $i$-th negative sample and
$\sigma(\cdot)$ is the sigmoid function.

\paragraph{Transitive Relation Loss.} Given a transitive relation $r$, with
embeddings $\vr = (r_1, r_2, r_3, r_4)$ and projection dimension $i$,
we optimize its projected embedding $\proj{\vr}{i} \in \mathbb{R}^{n-1}$ to approximate the
identity function. Thus:
\begin{equation}
  \label{eq:reg_loss}
  L_{\proj{\vr}{i}} = ||\proj{r_1}{i}-1||_1 + ||\proj{r_3}{i}-1||_1 +
  ||\proj{r_2}{i}||_1 + ||\proj{r_4}{i}||_1
\end{equation}

The modified distance function $\dist{box\_tr}$ together with the
transitive relation loss function $L_{\vr_i}$ provide guarantees that the
generated embeddings will preserve transitivity. That is, given a
transitive relation $r$, we have
$\forall a,b,c \in \mathcal{V}: r(a,b) \land r(b,c) \to r(a,c)$
(Theorem \ref{thm:transitivity}).

\begin{theorem}[Transitive Inference Property]
  \label{thm:transitivity}
  Let $(a,r,b), (b,r,c) \in \mathcal{E}$ with $r$ being a transitive
  relation, but $(a,r,c) \not\in \mathcal{E}$. Let embeddings
  $\query{\va}, \query{\vb}, \query{\vc}$ and
  $\answer{\va}, \answer{\vb}, \answer{\vc}$ for $a,b,c$ and relation
  embedding for $r$ be optimized using the loss function
  $L' = L + L_{\proj{\vr}{i}}$ until convergence. Let $i$ be the
  dimension chosen for $r$ in $\dist{ordering}$. If
  $\dist{box\_tr}(T_r(\query{\va}), \answer{\vb}) = 0$ and
  $\dist{box\_tr}(T_r(\query{\vb}), \answer{\vc}) = 0$, then
  $\dist{box\_tr}(T_r(\query{\va}), \answer{\vc}) = 0$ with
  $\answer{\vc}[i] < \answer{\vb}[i] < \answer{\va}[i]$.
\end{theorem}

\begin{proof}
  The minimization of $L_{\proj{\vr}{i}}$ ensures that the relation
  transformation $T_r$ approaches idempotency in dimension $i$, with
  $\proj{r_1}{i} \approx 1$, $\proj{r_3}{i} \approx 1$,
  $\proj{r_2}{i} \approx 0$, and $\proj{r_4}{i} \approx 0$.
  
  For the observed triples $(a,r,b)$ and $(b,r,c)$, the distance
  function $\dist{box\_tr}$ approaching zero implies:
  
  \begin{enumerate}
  \item For dimensions other than $i$, all answer entities lie
    within their corresponding query boxes after relation projection:
    $\answer{\vb} \in T_r(\query{\va})$ and
    $\answer{\vc} \in T_r(\query{\vb})$
  
    \item For dimension $i$, the ordering constraints are satisfied:
    $\cen{\answer{\vb}}[i] \leq \lowc{\query{\va}}[i] - \lambda$ and
    $\cen{\answer{\vc}}[i] \leq \lowc{\query{\vb}[i]} - \lambda$
  \end{enumerate}
  From the first constraint, and given that $T_r$ approaches an
  identity transformation in dimension $i$, we have
  $T_r(\query{\va})[i] \approx \query{\va}[i]$ and
  $T_r(\query{\vb})[i] \approx \query{\vb}[i]$.
  
  The idempotency property of box transformation in non-$i$ dimensions
  (Equation \ref{eq:reg_loss}) implies that if
  $\answer{\vb} \in T_r(\query{\va})$ and
  $\answer{\vc} \in T_r(\query{\vb})$, then
  $\answer{\vc} \in T_r(\query{\va})$ for all dimensions except
  $i$
  For dimension $i$, we have:
  $\cen{\answer{\vc}[i]} \leq \lowc{\query{\vb}[i]} - \lambda$ and
  $\cen{\answer{\vb}[i]} \leq \lowc{\query{\va}}[i] - \lambda$.
  
  Since every answer embedding is optimized to lie inside its query
  embedding, we get:
  $\cen{\answer{\vc}}[i] \leq \lowc{\query{\vb}}[i] - \lambda \leq
  \cen{\answer{\vb}}[i]$ and
  $\cen{\answer{\vb}}[i] \leq \lowc{\query{\va}}[i] - \lambda \leq
  \cen{\answer{\va}}[i]$. Thus,
  $\cen{\answer{\vc}}[i] \leq \lowc{\query{\va}}[i] - \lambda \leq
  \cen{\answer{\va}}[i]$, satisfying the ordering constraint for
  $(a,r,c)$. Therefore, we obtain
  $\dist{box\_tr}(\query{\va}, \answer{\vc}) \approx 0$ even though
  $(a,r,c) \not\in \mathcal{E}$, and the embeddings maintain the
  ordering
  $\vc_{\mbox{\scriptsize answer}}[i] < \vb_{\mbox{\scriptsize
      answer}}[i] < \va_{\mbox{\scriptsize answer}}[i]$.
\end{proof}

\section{Experimental Results}

\subsection{Experimental Settings}
We follow the same experimental procedure from previous methods
\cite{gqe,query2box,betae,cone} and evaluate query answering over
incomplete knowledge graphs.
\paragraph{Datasets.} We use three datasets: WN18RR-QA~\cite{line},
NELL-QA~\cite{deeppath} and FB15k-237
(FB237)~\cite{toutanova2015observed}. The datasets contain training
$\mathcal{G}_{\mbox{\scriptsize train}}$, validation
$\mathcal{G}_{\mbox{\scriptsize valid}}$ and testing
$\mathcal{G}_{\mbox{\scriptsize test}}$ graphs such that
$\mathcal{G}_{\mbox{\scriptsize train}} \subseteq
\mathcal{G}_{\mbox{\scriptsize valid}} \subseteq
\mathcal{G}_{\mbox{\scriptsize test}}$. For a query $q$, its answer
set follows the inclusions
$\llbracket q \rrbracket_{\mbox{\scriptsize train}} \subseteq
\llbracket q_{\mbox{\scriptsize valid}} \rrbracket \subseteq
\llbracket q_{\mbox{\scriptsize test}} \rrbracket $ and we take 
the set
$\llbracket q_{\mbox{\scriptsize test}} \rrbracket \backslash
\llbracket q_{\mbox{\scriptsize valid}} \rrbracket$ for testing. The
training and validation queries contain five types of positive
conjunctive queries ($1p,2p,3p,2i,3i$) and five types of queries with
negation ($2in,3in,inp,pin,pni$). Furthermore, the testing set
contains extra query types ($ip,pi,2up,up$) which are used to measure
the generalization capability of the method. We show additional
dataset information such as training/testing splits in
Appendix~\ref{app:datasets}.

\paragraph{Training protocol and baselines.} We trained our method
using gradient descent optimization and used the Adam
optimizer~\cite{adam}. We performed grid-search to select optimal
hyperparameters for our method. We show the selected hyperparameters
in Appendix~\ref{app:hyperparameters}.  For WN18RR-QA, we report the
results from~\cite{line} for all methods, except for ConE, for which
we reproduce the results. For NELL and FB237, we report the results
from~\cite{cone}. For the recent method AConE~\cite{acone}, we could
not find an available implementation, therefore we do not compare
against this method directly but provide a specific discussion for
this method. We also include NMP-QEM~\cite{long-etal-2022-neural}, a
neural method, with which we do not compare directly but use it as a
reference. Regarding our method, an important aspect to highlight is
the initial choice of answer embeddings. While we could in principle
associate an answer node with is own embedding, we also experimented
fixing the answer embedding to be the same as the center of the query
embedding of the same entity. For our main experiments, we report the
aggregated results from eight runs with different random seeds, which
we show in Appendix~\ref{app:stats}. Finally, every experiment took
around 20 hours on average and, throughout all experiments, we used
the following GPUs: NVIDIA GeForce GTX 1080 Ti, Quadro P6000, Tesla
P100 and Tesla V100.

\paragraph{Evaluation protocol.} For every query $q$ in the testing
set, we select every answer
$v \in \llbracket q \rrbracket_{\mbox{\scriptsize test}} \backslash
\llbracket q \rrbracket_{\mbox{\scriptsize valid}}$ and we rank it
against all entities
$\mathcal{V} \backslash \llbracket q \rrbracket_{\mbox{\scriptsize
    test}}$. Given the obtained rank $rank(v)$, we report Mean
Reciprocal Rank (MRR) as follows:

\begin{equation}
  \label{eq:mrr}
  \mbox{MRR}(q) = \frac{1}{|\llbracket q \rrbracket_{\mbox{\scriptsize test}} \backslash \llbracket q
    \rrbracket_{\mbox{\scriptsize valid}}|} \sum_{v \in \llbracket q \rrbracket_{\mbox{\tiny test}} \backslash \llbracket q
    \rrbracket_{\mbox{\tiny valid}}}{rank(v)}
\end{equation}

\subsection{Results}

\paragraph{RQ1: Can we rely on purely geometric operations to model
  complex KG queries?}

Table \ref{tab:rq1_12} (left part) shows the results on queries
without negation. We can see that \method{} outperforms baselines
in most metrics (except $3p$ in WN18RR and $2p,3p,pi,up$ in
FB237). More importantly, we observe that the intersection operation,
which baseline methods learn through a neural network
can be modeled by \method{} using only geometric intersection
operations, with outperforming results.

\begin{table}[tbp]
  \centering
  \caption{Results on several query answering datasets over positive
    and negative queries. Baseline results for WN18RR were obtained
    from LinE~\cite{line} and we reproduced ConE~\cite{cone}. Baseline
    results for NELL and FB237 were obtained from ConE~\cite{cone} and
    NPM-QEM~\cite{long-etal-2022-neural}. AConE results are obtained
    from \cite{acone}.}
  \adjustbox{width=0.98\textwidth}{
  \begin{tabular}{p{0.3cm}lrrrrrrrrr|rrrrr}
    \toprule
    
    \multicolumn{2}{c}{} & \multicolumn{9}{c|}{Positive queries} & \multicolumn{5}{c}{Negative queries} \\
    \cmidrule{2-16}
    \multicolumn{1}{c}{}  &   Model     &1p          &2p          &3p          &2i          &3i          &pi          &ip          &2u          &up & 2in        & 3in        & inp        & pin       & pni\\
    \cmidrule{2-16}
    \multirow{6}{*}{\rotatebox{90}{WN18RR}} &GQE & 18.0 & 4.5 & 2.7 & 19.3 & 23.9 & 9.9 & 10.6 & 2.3 & 3.5 & -- & -- & -- & -- & -- \\
&Q2B & 22.4 & 4.6 & 2.3 & 25.6 & 41.2 & 13.2 & 11.0 & 2.9 & 3.4 & -- & -- & -- & -- & -- \\
&\betae{} & 44.1 & 9.8 & 3.8 & 57.2 & 76.2 & 32.6 & 17.9 & 7.5 & 5.3 & \s{12.7} & 59.9 & 5.1 & 4.0 & 7.4 \\
&LinE & 45.1 & 12.3 & 6.7 & 47.1 & 67.1 & 24.8 & 14.7 & 8.4 & 6.9 & 12.5 & 60.8 & 7.3 & 5.2 & \s{7.7} \\
&ConE & \s{46.8} & \s{17.1} & \f{12.4} & \s{58.9} & \s{86.7} & \s{33.9} & \s{20.8} & \s{8.8} & \s{12.2} & \s{12.7} & \s{62.3} & \s{13.2} & \s{7.7} & 7.6 \\
&\method{} & \f{52.3} & \f{18.5} & \s{11.5} & \f{66.8} & \f{88.9} & \f{39.5} & \f{25.8} & \f{13.2} & \f{13.3} & \f{16.4} & \f{68.3} & \f{13.4} & \f{8.0} & \f{9.4} \\
    \cmidrule{3-16}
                         & \g{AConE} & \g{50.9} & \g{17.6} & \g{9.9} & \g{70.5} & \g{89.0} & \g{38.9} & \g{29.6} & \g{18.4} & \g{14.0} & \g{18.3} & \g{70.1} & \g{13.4} & \g{7.6} & \g{10.4}\\
                         & \g{NMP-QEM} & \g{53.1} & \g{24.3} & \g{14.1} & \g{68.5} & \g{86.6} & \g{38.2} & \g{19.2} & \g{12.7} & \g{13.2}  & \g{24.2} & \g{68.2} & \g{19.8} & \g{12.0} & \g{16.3} \\
    \cmidrule{2-16}
    \multirow{5}{*}{\rotatebox{90}{NELL}}   & GQE       & 33.1     & 12.1     & 9.9      & 27.3     & 35.1     & 18.5     & 14.5     & 8.5      & 9.0      &--        &--        &--        &--       &--\\
                                            & Q2B       & 42.7     & 14.5     & 11.7     & 34.7     & 45.8     & 23.2     & 17.4     & 12.0     & 10.7     &--        &--        &--        &--       &--\\
                                            & \betae{}  & 53.0     & 13.0     & 11.4     & 37.6     & 47.5     & 24.1     & 14.3     & 12.2     & 8.5      & 5.1      & 7.8      & 10.0     & 3.1     & 3.5 \\  
                                            & ConE      & \s{53.1} & \s{16.1} & \s{13.9} & \s{40.0} & \s{50.8} & \s{26.3} & \s{17.5} & \s{15.3} & \s{11.3} & \s{5.7}  & \s{8.1}  & \s{10.8} & \s{3.5} & \s{3.9}\\
                                            & \method{} & \f{58.2} & \f{17.6} & \f{15.6} & \f{40.8} & \f{51.5} & \f{27.7} & \f{19.6} & \f{16.1} & \f{12.8} & \f{6.2}  & \f{8.4}  & \f{11.0} & \f{4.0} & \f{4.3}\\
    \cmidrule{3-16}
                         & \g{AConE}     & \g{54.5} & \g{17.7} & \g{14.4} & \g{41.9} & \g{53.0} & \g{26.1} & \g{20.7} & \g{16.5} & \g{12.8} & \g{5.2} & \g{7.7} & \g{9.4} & \g{3.2} & \g{3.7} \\
                         & \g{NMP-QEM} & \g{68.8} & \g{23.9} & \g{17.8} & \g{47.0} & \g{55.0} & \g{31.1} & \g{26.0} & \g{29.6} & \g{20.7} & \g{10.0} & \g{9.2} & \g{12.9} & \g{4.8} & \g{7.4}\\
    \cmidrule{2-16}
    \multirow{5}{*}{\rotatebox{90}{FB237}}  &   GQE     & 35.2     & 7.4      & 5.5      & 23.6     & 35.7     & 16.7     & 8.4      & 5.8      & 4.6      &--        &--        &--        &--       &--\\
                                            & Q2B       & 41.3     & 9.9      & 7.2      & 31.1     & 45.4     & 21.9     & 11.9     & 9.3      & 7.1      &--        &--        &--        &--       &--\\
                                            & \betae{}  & 39.0     & 10.9     & 7.6      & 30.5     & 42.3     & 21.4     & 11.6     & 9.0      & \s{10.2} & 5.1      & 7.9      & 7.4      & 3.6     & 3.4 \\    
                                            & ConE      & \s{41.8} & \f{12.8} & \f{11.0} & \s{32.6} & \s{47.3} & \f{25.5} & \s{14.0} & \s{12.5} & \f{10.8} & \s{5.4}  & \s{8.6}  & \f{7.8}  & \s{4.0} & \f{3.6} \\
                                            & \method{} & \f{44.2} & \s{11.6} & \s{9.7}  & \f{33.8} & \f{48.6} & \s{25.2} & \f{14.8} & \f{14.1} & 10.0     & \f{6.0}  & \f{11.2} & \s{7.7}  & \f{4.6} & \f{3.6} \\
    \cmidrule{3-16}
                         &     \g{NMP-QEM} & \g{46.2} & \g{12.9} & \g{11.3} & \g{35.0} & \g{47.8} & \g{25.6} & \g{15.1} & \g{15.0} & \g{10.9} & \g{6.8} & \g{11.7} & \g{8.2} & \g{5.5} & \g{4.7} \\
    \bottomrule
  \end{tabular}
  }
    \label{tab:rq1_12}
\end{table}

Additionally, Table \ref{tab:rq1_12} (right part) shows the results on
queries that include the negation operator. We can see that the
exclusion term for the negative component of the query is effective
and enables our method to outperform baselines for all queries and
datasets, (except $inp$ in FB237). When comparing with a
neural-method, \method{} approaches the performance of NMP-QEM in
WN18RR and FB237. Nevertheless, \method{} provides a direct geometric
interpretation of the embeddings and operators.

Regarding AConE ~\cite{acone}, we could not reproduce or test the
method on every dataset due to lack of implementation.  However, when
comparing with AConE published results, we find that GeometrE performs
mostly the same as AConE in WN18RR and NELL. In relational queries (1p,2p,3p) both methods have
geometric operations and GeometrE surpasses AConE in WN18RR (27.4
vs. 26.1) and NELL (30.5 vs. 28.9) in average MRR. For intersection
queries (2i,3i,pi,ip), where AConE employs a neural component,
GeometrE underperforms AConE in WN18RR (55.2 vs. 57.0) and NELL (34.9
vs. 35.4) in average MRR revealing a trade-off for maintaining full
geometric interpretability instead of relying on black-box neural
networks. For negation queries (2in,3in,inp,pin,pni), GeometrE is
underperforms AConE in WN18RR (23.1 vs. 24.0) but outperforms in NELL
(6.8 vs. 5.8). Therefore, GeometrE keeps mostly the same performance
as AConE, while preserving geometric interpretability.

\paragraph{RQ2: Does a transitive loss improve multi-hop reasoning capabilities?}

To answer whether the proposed transitive loss function can enhance
the multi-hop reasoning capabilities, we evaluate \method{} on WN18RR
and report the aggregated results (mean and standard deviation) of
eight experiments with different random seeds
(Table~\ref{tab:transitive}). For this evaluation, we considered
relations \texttt{hypernym} and \texttt{has\_part} from WN18RR as
transitive~\cite{line}.
The global evaluation over WN18RR graph show a slight improvement
(statistically non significant) in the presence of a transitive loss
function across 1-hop, 2-hop and 3-hop queries.

\begin{table}[tbp]
  \centering
  \caption{Transitive vs non-transitive loss on the WN18RR dataset}
  \begin{adjustbox}{width=0.98\textwidth}
    \begin{tabular}{lccccccccc}
      \toprule
      \multirow{2}{*}{Model} & \multicolumn{4}{c|}{1-hop} & \multicolumn{4}{c|}{2-hop} & \multicolumn{1}{r}{3-hop}\\
      \cmidrule{2-10}
                             &1p                               & ip                               &2in                              & pni                       &2p                         &up                         & inp                        &pin                       &3p           \\
\method{}-notr & 52.2{\scriptsize $\pm$0.1 } & 25.8{\scriptsize $\pm$0.2 } & \f{16.4}{\scriptsize $\pm$0.2 } & \f{9.4}{\scriptsize $\pm$0.1 } & 18.5{\scriptsize $\pm$0.1 } & 13.1{\scriptsize $\pm$0.2 } & 13.2{\scriptsize $\pm$0.2 } & 7.8{\scriptsize $\pm$0.2 } & 11.3{\scriptsize $\pm$0.2 }  \\
      \method{}-tr & \f{52.3}{\scriptsize $\pm$0.1 } & \f{25.9}{\scriptsize $\pm$0.2 } & 16.3{\scriptsize $\pm$0.2 } & \f{9.4}{\scriptsize $\pm$0.1 } & \f{18.6}{\scriptsize $\pm$0.1 } & \f{13.4}{\scriptsize $\pm$0.2 } & \f{13.4}{\scriptsize $\pm$0.1 } & \f{8.1}{\scriptsize $\pm$0.1 } & \f{11.5}{\scriptsize $\pm$0.1 }  \\
      \bottomrule
  \end{tabular}
  \end{adjustbox}
  \label{tab:transitive}
\end{table}

Therefore, we investigated the number of queries involving composition
of transitive relations in WN18RR. We found that, for 2-hop and 3-hop
queries, the percentages of \texttt{training/validation/testing}
queries are low: 2p (5.3/1.9/2.3), up (0.0/1.5/0.7), inp (1.5 /2.2
/2.3), pin (4.1 /2.2 /2.0), 3-3p (3.5/2.6/2.1). This constitutes a
potential reason for the negligible improvements in the general
benchmark evaluation.

For that reason, we conducted a targeted evaluation to verify our
claim that GeometrE preserves the transitivity formula $\forall a,b,c:
r(a,b) \land r(b,c) \to r(a,c)$. By constructing transitive chains, we
showed that GeometrE achieves significantly better ordering
preservation (Spearman correlation 0.87 vs 0.12) compared to models
without the transitive loss.

However, the number of transitive relations considered is small and
the impact of the transitive loss function might be overlooked in the
overall evaluation. For this reason, we specifically analyze
transitive relations. We computed the transitive chains in the
training set from relations \texttt{+\_hypernym}, and
\texttt{+\_has\_part} (we omit \texttt{-\_hypernym} and
\texttt{-\_has\_part}\footnote{``+'' and ``-'' in the relation name
  indicates the relation name or its inverse. Due to distance
  functions in Equations \ref{eq:ordering} and
  \ref{eq:ordering_inverse}, a transitive relation and its inverse
  share the projection dimension.}). For example, for triples
$(a,r,b),(b,r,c),\allowbreak(c,r,d),(c,r,e)$ the following chains are
generated: $(a,b,c,d)$ and $(a,b,c,e)$. Then, for each relation $r$,
we find the value
$\mbox{emb}(a) = \va_{\mbox{\scriptsize answer}}[i]$, where $i$ is the
transitive dimension. Ideally, for a chain $(a,b,c,d)$ we expect
$\mbox{emb}(a) > \mbox{emb}(b) > \mbox{emb}(c) > \mbox{emb}(d)$. We
compute the Spearman correlation score between the expected order of
sequences and the order found in the embeddings and we show the scores
in Table \ref{tab:spearman}. We can see that the order of embeddings
is better preserved in the presence of the transitive
loss. Additionally, we show in Figure \ref{fig:chains} the transitive
chains, where for a chain $a_1, \ldots, a_n$ if
$\mbox{emb}(a_i) > \mbox{emb}(a_{i+1})$ (preserved) we plot a green
line, otherwise we plot a red line. We can see that the visualization
is consistent with the reported Spearman scores.

\begin{table}[tbp]
  \centering
  \caption{Spearman correlation score between transitive chains and embeddings  }
  \begin{tabular}{lrr}
    \toprule
    Model & hypernym & has\_part\\
    \midrule
    \# of chains & 23991 & 2878 \\
    \midrule
    \method{}-notr & {0.12} & {0.17} \\
    \method{}-tr & \f{0.87} & \f{0.79} \\

      \bottomrule
  \end{tabular}
  \label{tab:spearman}
\end{table}

\begin{figure}[tbp]
  \Description{Representation of transitive chains in latent space}
    \centering
    \includegraphics[width=0.7\textwidth]{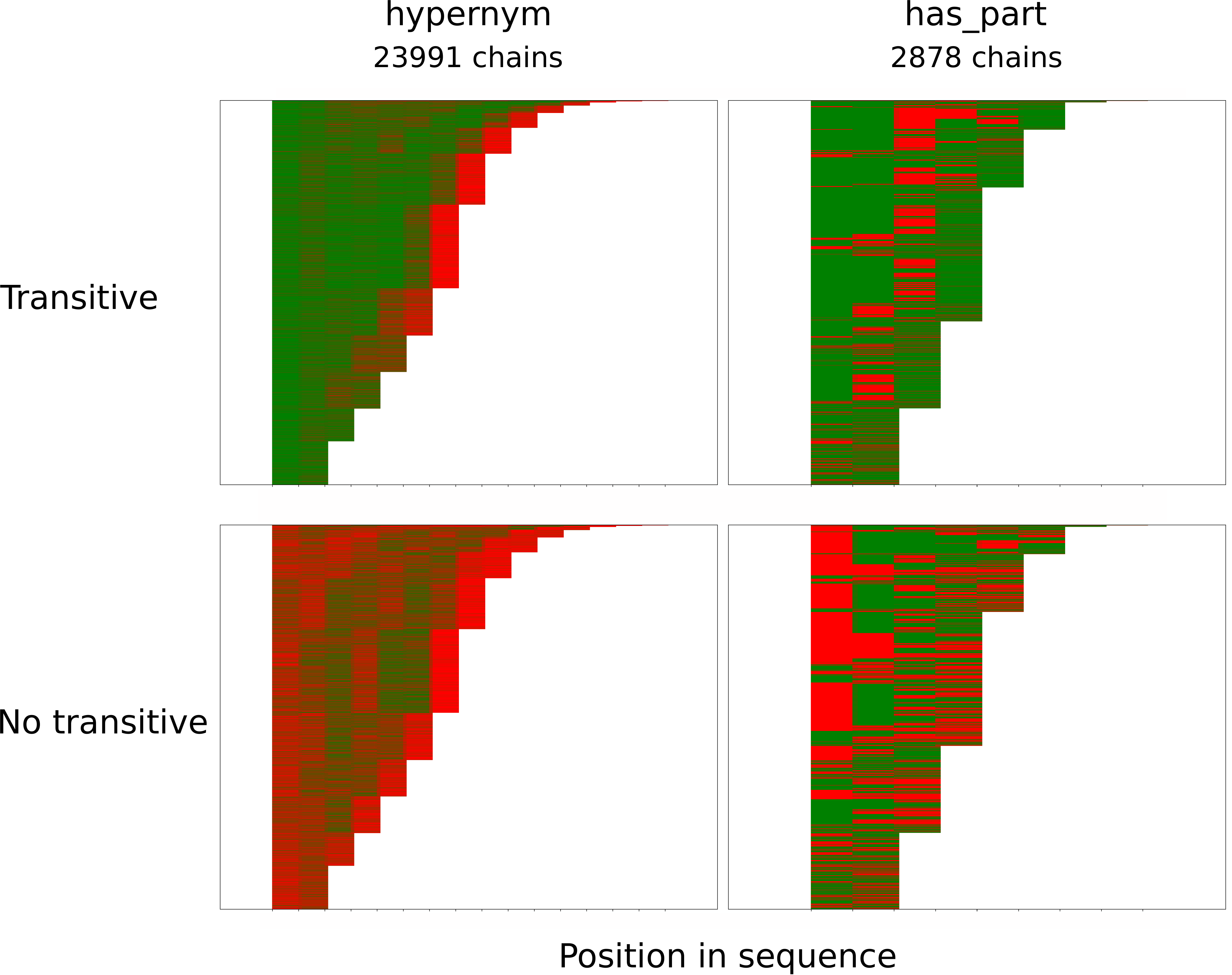}
    \caption{Transitive chains. For a chain $a_1,\ldots,a_n$ we plot a
      green line if the embeddings $emb(a_i)>emb(a_{i+1})$, otherwise
      we plot a red line. Upper plots show the embeddings generated
      using the transitive loss and the lower plots show the
      embeddings generated using the normal loss function.}
    \label{fig:chains}
  \end{figure}

\paragraph{Ablation Study}

We used a box model similar to Query2Box~\cite{query2box}. While
our method relies on purely geometrical operations, we have added a
more complex function to compute relation projection. For every
relation $r$, its embedding is $\vr = (r_1,r_2,r_3,r_4)$, where
$r_1,r_3$ are multiplicative factors and $r_2,r_4$ are additive
terms. In Query2Box, the embedding of the relation is equivalent to
$\vr_{\mbox{\scriptsize Q2B}} = (\vone, r_2, \vone, r_3)$. In
Table~\ref{tab:ablation} we show the performance results of our method
when we fix $\vr_{\mbox{\scriptsize mul}}=(r_1, \vzero, r_3, \vzero)$
(\method{}-m), and when we fix
$\vr_{\mbox{\scriptsize add}} = (\vone, r_2, \vone, r_3)$
(\method{}-a). Additionally, we add the results from Query2Box. We can
see that, as expected, the performance of Query2Box is similar to
\method{}-a in the NELL and FB237 datasets. Furthermore, our results
also show that the combination of additive and multiplicative
components in the relation embedding improve their quality which is
evidenced by the improvement of prediction performance.
Transversely, we also study two ways of initializing the answer
embedding. For a node $a$ in the knowledge graph, we represent it as a
query $\va_{\scriptsize query}=(\cen{a},\off{a})$ and as answer
$\va_{\scriptsize answer} = (\cen{a}', \vzero)$, we investigate if it
is better to create a separate vector for $\cen{a}'$ or we can reuse
some vector (i.e., $\cen{a}' = \cen{a}$). For WN18RR we found that
using different embeddings for answers is more suitable, whereas for
NELL the results are mixed and for FB237 using the center of the query
embeddings as answer embeddings obtains better performance (Table
\ref{tab:answer_embeddings}). This can be caused by the sparsity level
of each graph. For example, FB237 is the densest graph, NELL is less
dense and finally WN18RR is sparse~\cite{lin-etal-2018-multi}. Sparser
graphs require separating answer and entity embeddings while denser
graphs work better with a strong coupling between answer and entity
embeddings.

\begin{table}[tbp]
  \centering
  \caption{Ablation of multiplicative and additive components in the
    relation projection operation.}
\begin{adjustbox}{width=0.99\textwidth}
  \begin{tabular}{llrrrrrrrrrrrrrr}
    \toprule
    Dataset                 & Model            &1p          &2p          &3p          & 2i & 3i & pi & ip          & 2u & up & 2in, &3in & inp & pin & pni\\
    \midrule
    \multirow{4}{*}{WN18RR} & Q2B & 22.4 & 4.6 & 2.3 & 25.6 & 41.2 & 13.2 & 11.0 & 2.9 & 3.4 & 0.0 & 0.0 & 0.0 & 0.0 & 0.0 \\
                            & \method{}-a & 36.6 & 6.7 & 4.7 & 53.2 & 84.4 & 25.8 & 11.9 & 4.7 & 3.6 & 8.6 & 44.0 & 4.4 & 3.5 & 5.6 \\
                            & \method{}-m & \s{51.2} & \s{17.2} & \s{10.3} & \s{64.6} & \s{88.5} & \s{39.1} & \f{26.4} & \s{12.5} & \s{12.1} & \f{16.9} & \f{68.3} & \s{12.3} & \s{7.7} & \s{8.9} \\
                            & \method{}-tr & \f{52.4} & \f{18.6} & \f{11.5} & \f{67.2} & \f{89.0} & \f{39.5} & \s{25.9} & \f{13.2} & \f{13.4} & \s{16.3} & \s{68.0} & \f{13.3} & \f{8.1} & \f{9.4} \\
    \midrule
    \multirow{5}{*}{NELL} & Q2B & 42.7 & 14.5 & 11.7 & 34.7 & 45.8 & 23.2 & \s{17.4} & 12.0 & 10.7 & 0.0 & 0.0 & 0.0 & 0.0 & 0.0 \\
                            & \method{}-a & 42.4 & 14.8 & \s{12.5} & 34.2 & 44.7 & 23.8 & 17.2 & 12.5 & 10.9 & 5.0 & 6.7 & 8.1 & 3.4 & 3.3 \\
                            & \method{}-m & \s{55.1} & \s{15.4} & 10.7 & \s{39.6} & \s{49.5} & \s{26.2} & \s{17.4} & \s{15.9} & \s{11.2} & \f{7.0} & \f{9.6} & \s{9.0} & \s{3.9} & \s{4.1} \\
                            & \method{} & \f{58.2} & \f{17.6} & \f{15.6} & \f{40.8} & \f{51.5} & \f{27.7} & \f{19.6} & \f{16.1} & \f{12.8} & \s{6.2} & \s{8.4} & \f{11.0} & \f{4.0} & \f{4.3} \\
    \midrule
    \multirow{5}{*}{FB237}  & Q2B & \s{41.3} & \s{9.9} & \s{7.2} & \s{31.1} & \s{45.4} & 21.9 & 11.9 & 9.3 & 7.1 & 0.0 & 0.0 & 0.0 & 0.0 & 0.0 \\
                            & \method{}-a & 41.1 & 9.2 & 6.5 & 29.3 & 43.4 & 21.9 & \s{13.1} & 11.6 & 7.2 & \s{5.4} & \s{9.4} & \s{5.9} & \s{3.6} & 2.9 \\
                            & \method{}-m & 40.6 & 9.1 & 5.7 & 30.9 & 45.0 & \s{22.0} & 12.3 & \s{12.8} & \s{7.4} & 5.3 & 9.3 & \s{5.9} & 3.3 & \s{3.0} \\
                            & \method{} & \f{44.2} & \f{11.6} & \f{9.7} & \f{33.8} & \f{48.6} & \f{25.2} & \f{14.8} & \f{14.1} & \f{10.0} & \f{6.0} & \f{11.2} & \f{7.7} & \f{4.6} & \f{3.6} \\

    \bottomrule
  \end{tabular}
  \end{adjustbox}
  \label{tab:ablation}
\end{table}

\begin{table}[htbp]
  \small
  \centering
  \caption{Analysis on the choice of answer embeddings. \method{}-na
    refers to using the center of the query of a node $a$ as the
    embedding of the answer. \method-wa refers to using a separate
    vector to model the answer embedding.}
\begin{adjustbox}{width=0.99\textwidth}
  \begin{tabular}{llrrrrrrrrrrrrrr}
    \toprule
    Dataset& Model&1p&2p&3p&2i&3i&pi&ip&2u&up&2in&3in&inp&pin&pni \\
    \midrule
    \multirow{2}{*}{WN18RR} & \method{}-na & \s{51.4} & \s{17.1} & \s{10.1} & \s{65.8} & \s{88.2} & \s{37.5} & \f{26.9} & \s{10.5} & \s{12.7} & \s{15.2} & \s{66.9} & \s{12.0} & \s{7.5} & \s{8.8} \\
& \method{}-wa & \f{52.3} & \f{18.6} & \f{11.5} & \f{67.2} & \f{89.0} & \f{39.5} & \s{25.9} & \f{13.2} & \f{13.4} & \f{16.3} & \f{68.0} & \f{13.4} & \f{8.1} & \f{9.4} \\
    \midrule
    \multirow{2}{*}{NELL} & \method{}-na & \f{58.2} & \f{17.6} & \f{15.6} & \s{40.8} & \s{51.5} & \f{27.7} & \f{19.6} & \s{16.1} & \f{12.8} & \f{6.2} & \s{8.4} & \f{11.0} & \f{4.0} & \f{4.3} \\
& \method{}-wa & \s{57.4} & \s{16.0} & \s{14.9} & \f{41.1} & \f{51.6} & \s{26.7} & \s{17.2} & \f{16.8} & \s{11.9} & \f{6.2} & \f{8.8} & \s{10.7} & \f{4.0} & \f{4.3} \\
    \midrule
    \multirow{2}{*}{FB237} & \method{}-na & \f{44.2} & \f{11.6} & \f{9.7} & \s{33.8} & \f{48.6} & \f{25.2} & \f{14.8} & \f{14.1} & \f{10.0} & \f{6.0} & \f{11.2} & \f{7.7} & \s{4.6} & \f{3.6} \\
& \method{}-wa & \s{42.6} & \s{11.2} & \s{9.3} & \f{33.9} & \s{48.5} & \s{24.8} & \s{13.2} & \s{13.6} & \s{9.5} & \s{5.7} & \s{11.1} & \s{7.4} & \f{4.8} & \s{3.5} \\
    \bottomrule
  \end{tabular}
\end{adjustbox}
  \label{tab:answer_embeddings}
\end{table}

\section{Discussion}
We introduced \method{}, a box model method for multi-hop reasoning on
knowledge graphs. We demonstrated that complex queries can be
constructed without learning logical operators. Instead, under
appropriate conditions, learning the query embeddings themselves is
sufficient. In our approach, these ``appropriate conditions'' involve
using a combination of additive and multiplicative components for
relation projection, and providing a tailored scoring function for
negation queries.

We also developed a framework to effectively model transitive
relations. While we employed a simple idempotent transformation (the
identity function), we enhanced it with a specialized loss function
that preserves the ordering nature of transitive chains. Our empirical
results confirm that this approach successfully captures the
sequential structure inherent in transitive relationships. A
limitation of GeometrE is that it reserves one dimension per
transitive relation, meaning that simply increasing the embedding size
($n$) might not scale if the number of transitive relations ($m$) is
very large ($m >> n$).  However, under some conditions, dimensions can
be reused. If two transitive relations have disjoint domains and
ranges, they can share the same ordering dimension without
interference. While this approach relies on domain/range knowledge
(the disjointness condition), it relaxes the constraint $m \le n$ and
allows the method to scale to knowledge graphs with more transitive
relations than dimensions.

Furthermore, our method can outperform existing
geometric baselines and approach the performance of purely neural-based
methods, while providing geometric interpretability. A promising
direction for future research would be to investigate different
transformation types and loss functions in alternative geometric
spaces to further improve the representation of relational properties.

\section{Conclusion}
We present \method{}, a method to perform multi-hop reasoning by
embedding queries to boxes and mapping logical operators to purely
geometric operations. We additionally introduced a transitive loss
function that preserves the ordered structure of transitive
relationships. Our experiments show that \method{} outperforms  baselines
on diverse types of knowledge graph queries.

\begin{credits}

\subsubsection{\ackname}
This work has been supported by funding from King Abdullah University
of Science and Technology (KAUST) Office of Sponsored Research (OSR)
under Award No. URF/1/4675-01-01, URF/1/4697-01-01, URF/1/5041-01-01,
REI/1/5659-01-01, and REI/1/5235-01-01, and by funding from King
Abdullah University of Science and Technology (KAUST) -- KAUST Center
of Excellence for Smart Health (KCSH), under award number 5932.  We
acknowledge support from the KAUST Supercomputing Laboratory.

\subsubsection{Supplemental Material Statement:}

To reproduce our results, we make our code available at
\url{https://github.com/bio-ontology-research-group/geometrE}. The WN18RR-QA
dataset was obtained from
\url{https://github.com/nelsonhuangzijian/WN18RR-QA} and both NELL-QA
and FB15k-237 were obtained from
\url{https://github.com/snap-stanford/KGReasoning}

\subsubsection{Generative AI statement}

The authors used Generative AI tools to refine the sentence structure,
format mathematical proofs, and improve the readability of this
manuscript. The authors affirm that the scientific content, original
drafting, and final verification of all materials are their own work
and take full responsibility for the integrity of the publication.

\subsubsection{\discintname}
The authors have no competing interests to declare that are relevant
to the content of this article.
\end{credits}
%
%
%
\bibliographystyle{splncs04}
\bibliography{refs}

\clearpage

\appendix

\section{Complement Intersection}
\label{app:complement_intersection}
\begin{lemma}[Complement Intersection]

Let $B_1$ and $B_2$ be two $n$-dimensional boxes such that
$B_1 \cap B_2 = \emptyset$. Also, we define the complement of $B_1$, denoted
as $\bar{B}_1$, as:
\begin{equation*}
  \bar{B}_1 = \{\vx \in \mathbb{R}^n : \vx \notin B_1\}
\end{equation*}

Then, the intersection of the complement
of $B_1$, denoted as $\bar{B}_1$, with $B_2$ is $B_2$ itself:
\begin{equation*}
\bar{B}_1 \cap B_2 = B_2
\end{equation*}
\end{lemma}

\begin{proof}
  Let $B_1 = (\lowc{B_1}, \upc{B_1})$ and
  $B_2 = (\lowc{B_2}, \upc{B_2})$ be two $n$-dimensional boxes. Since
  $B_1 \cap B_2 = \emptyset$, there exists at least one dimension
  $j \in \{1, 2, \ldots, n\}$ such that either
  $\upc{B_1}[j] < \lowc{B_2}[j]$ or $\upc{B_2}[j] < \lowc{B_1}[j]$.

  This means that for any point $\vp \in \mathbb{R}^n$,
  $\vp \in \bar{B}_1$ if and only if there exists at least one
  dimension $i \in \{1, 2, \ldots, n\}$ such that
  $\vp[i] < \lowc{B_1}[i]$ or $\vp[i] > \upc{B_1}[i]$. For any point
  $\vq \in B_2$, we have $\lowc{B_2}[i] \leq \vq[i] \leq \upc{B_2}[i]$
  for all $i \in \{1, 2, \ldots, n\}$. Since
  $B_1 \cap B_2 = \emptyset$, for any point $\vq \in B_2$,
  $\vq \notin B_1$, which means $\vq \in \bar{B}_1$. Therefore, every
  point in $B_2$ is also in $\bar{B}_1$, implying that
  $B_2 \subseteq \bar{B}_1 \cap B_2$. Conversely, any point in
  $\bar{B}_1 \cap B_2$ must be in $B_2$, so
  $\bar{B}_1 \cap B_2 \subseteq B_2$.

  Combining both directions, we have $\bar{B}_1 \cap B_2 = B_2$.
\end{proof}

\section{Datasets}
\label{app:datasets}

In Table~\ref{tab:datasets} we show the number of queries per query
type in each dataset across training, validation and testing subsets.
\begin{table}[htbp]
\centering
\caption{Query statistics for WN18RR-QA}
\label{tab:datasets}
\begin{adjustbox}{width=0.9\textwidth}
\begin{tabular}{lrrrrrrrrrrrrrrr}
\toprule
  Split   &     1p &     2p &    3p &     2i &   3i &     pi &    ip &   2u &   up &   2in &   3in &   inp &   pin &   pni &   Total \\
  \midrule
  \midrule
          &   \multicolumn{14}{c}{WN18RR} \\
  \cmidrule{2-16}
Train   & 103509 & 103509 & 103509 & 103509 & 103509 &    0 &    0 &        0 &        0 & 10350 & 10350 & 10350 & 10350 & 10350 &  569295 \\
 Valid   &   5202 &   1000 &   1000 &   1000 &   1000 & 1000 & 1000 &     1000 &     1000 &  1000 &  1000 &  1000 &  1000 &  1000 &   18202 \\
 Test    &   5356 &   1000 &   1000 &   1000 &   1000 & 1000 & 1000 &     1000 &     1000 &  1000 &  1000 &  1000 &  1000 &  1000 &   18356 \\
  \midrule
  \midrule
          &   \multicolumn{14}{c}{NELL} \\
  \cmidrule{2-16}
  Train   & 107982 & 107982 & 107982 & 107982 & 107982 &    0 &    0 &        0 &        0 & 10798 & 10798 & 10798 & 10798 & 10798 &  593900 \\
 Valid   &  16910 &   4000 &   4000 &   4000 &   4000 & 4000 & 4000 &     4000 &     4000 &  4000 &  4000 &  4000 &  4000 &  4000 &   68910 \\
 Test    &  17021 &   4000 &   4000 &   4000 &   4000 & 4000 & 4000 &     4000 &     4000 &  4000 &  4000 &  4000 &  4000 &  4000 &   69021 \\
  \midrule
  \midrule
          &   \multicolumn{14}{c}{FB237} \\
  \cmidrule{2-16}
  Train   & 149689 & 149689 & 149689 & 149689 & 149689 &     0 &     0 &        0 &        0 & 14968 & 14968 & 14968 & 14968 & 14968 &  823285 \\
 Valid   &  20094 &   5000 &   5000 &   5000 &   5000 &  5000 &  5000 &     5000 &     5000 &  5000 &  5000 &  5000 &  5000 &  5000 &   85094 \\
 Test    &  22804 &   5000 &   5000 &   5000 &   5000 &  5000 &  5000 &     5000 &     5000 &  5000 &  5000 &  5000 &  5000 &  5000 &   87804 \\
   \bottomrule
\end{tabular}
\end{adjustbox}
\end{table}

\section{Hyperparameters selection}
\label{app:hyperparameters}

To train \method{}, we performed grid search over the following
hyperparameters: (a)$\alpha$: $[0, 0.1, 0.2, 0.5]$, (b) $\gamma$:
$[10, 20, 40]$, (c) embedding size: $[100, 200, 400]$, (d)
learning rate: $[0.001, 0.0005, 0.0001]$, (e) answer embedding:
[``yes'',``no'']. For the transitive loss function, we used
$\lambda=0.1$.

\begin{table}[h]
\centering
\caption{Hyperparameters for Knowledge Graph Datasets}
\begin{tabular}{lccc}
\toprule
\textbf{Hyperparameter} & \textbf{WN18RR} & \textbf{NELL} & \textbf{FB15k-237} \\
\midrule
$\alpha$ & 0.5 & 0.2 & 0.2 \\
$\gamma$ & 20 & 10 & 20 \\
embedding size & 400 & 400 & 400 \\
learning rate & 0.001 & 0.0005 & 0.0005 \\
with answer embedding & yes & no & no \\
\bottomrule
\end{tabular}
\end{table}

\section{Statistical significance test}
\label{app:stats}
In this section we show the experiments run over 7 different random
seeds. Table ~\ref{tab:stats} shows the results for positive queries
and Table ~\ref{tab:stats_neg} shows the results for negative queries.
\begin{table}[htbp]
  \small
  \centering
  \caption{Mean and standard deviation of experiments over 7 different
  random seeds on positive queries.}
\begin{adjustbox}{width=0.9\textwidth}
  \begin{tabular}{llrrrrrrrrr}
    \toprule
    Dataset&Model&1p&2p&3p&2i&3i&pi&ip&2u&up\\
        \midrule
    \multirow{2}{*}{WN18RR}& \method{}-notr & 52.2\std{0.1} & 18.5\std{0.1} & 11.3\std{0.2} & 68.1\std{0.3} & 88.9\std{0.3} & 40.1\std{0.3} & 25.8\std{0.2} & 12.7\std{0.2} & 13.1\std{0.2}  \\
           & \method{}-tr & 52.3\std{0.1} & 18.6\std{0.1} & 11.5\std{0.1} & 67.2\std{0.3} & 89.0\std{0.3} & 39.5\std{0.3} & 25.9\std{0.2} & 13.2\std{0.3} & 13.4\std{0.2}  \\
    NELL & \method{} & 58.2\std{0.1} & 17.6\std{0.1} & 15.6\std{0.1} & 40.8\std{0.2} & 51.5\std{0.2} & 27.7\std{0.1} & 19.5\std{0.1} & 16.1\std{0.1} & 12.8\std{0.1} \\
           FB237 &  \method{} & 44.2\std{0.1} & 11.6\std{0.1} & 9.7\std{0.1} & 33.7\std{0.2} & 48.6\std{0.2} & 25.2\std{0.2} & 14.8\std{0.1} & 14.1\std{0.1} & 10.0\std{0.1}  \\
    
    \bottomrule
  \end{tabular}
\end{adjustbox}
\label{tab:stats}
\end{table}
\begin{table}[b]
  \small
  \centering
  \caption{Mean and standard deviation of experiments over 7 different
  random seeds on negation queries.}
\begin{adjustbox}{width=0.6\textwidth}
  \begin{tabular}{llrrrrr}
    \toprule
    Dataset&Model&2in&3in&inp&pin&pni\\
    \midrule
    \multirow{2}{*}{WN18RR}&\method{}-notr & 16.4\std{0.2} & 68.1\std{0.8} & 13.2\std{0.2} & 7.8\std{0.2} & 9.4\std{0.1}\\
           &\method{}-tr & 16.3\std{0.2} & 68.0\std{0.5} & 13.4\std{0.1} & 8.1\std{0.1} & 9.4\std{0.1}\\
\midrule
    {NELL}&\method{} & 6.2\std{0.1} & 8.5\std{0.1} & 11.0\std{0.1} & 4.0\std{0.1} & 4.3\std{0.1}  \\
\midrule
    {FB237}&\method{} & 6.0\std{0.1} & 11.2\std{0.1} & 7.7\std{0.1} & 4.6\std{0.1} & 3.6\std{0.1} \\

\bottomrule
  \end{tabular}
\end{adjustbox}
\label{tab:stats_neg}
\end{table}

\end{document}